\documentclass{article}

\usepackage[utf8]{inputenc} 
\usepackage[T1]{fontenc}    
\usepackage{hyperref}       
\usepackage{url}            
\usepackage{booktabs}       
\usepackage{amsfonts}       
\usepackage{nicefrac}       
\usepackage{microtype}      
\usepackage{fullpage}

\usepackage{mathtools}
\usepackage{comment}
\usepackage{bm}
\usepackage{bbm}
\usepackage{algorithm,algorithmic}
\usepackage[algo2e,ruled]{algorithm2e}
\usepackage{subfigure} 
\usepackage{amsthm}
\usepackage{xcolor}
\usepackage{enumitem}
\setenumerate{leftmargin=12pt, noitemsep,topsep=0pt,parsep=0pt,partopsep=0pt}
\setitemize{leftmargin=12pt, noitemsep,topsep=0pt,parsep=0pt,partopsep=0pt}

\hypersetup{
     colorlinks   = true,
     urlcolor    = red,
     citecolor = blue,
     linkcolor = red
}

\title{Practical Algorithms for Best-K Identification in\\ Multi-Armed Bandits}

\usepackage{authblk}
\author[1]{Haotian Jiang}
\author[2]{Jian Li}
\author[2]{Mingda Qiao}
\affil[1]{Department of Physics, Tsinghua University}
\affil[2]{Institute for Interdisciplinary Information Sciences, Tsinghua University}

\newtheorem{lemma}{Lemma}[section]
\newtheorem{theorem}{Theorem}[section]

\newcommand{\High}{\mathrm{High}}
\newcommand{\Low}{\mathrm{Low}}
\newcommand{\eps}{\epsilon}
\newcommand{\empmean}[2]{\widehat\mu_{#1, #2}}  
\newcommand{\muhat}{\widehat\mu}
\newcommand{\mutil}{\widetilde\mu}
\newcommand{\event}{\mathcal{E}}
\newcommand{\ceps}{c_{\eps}}

\newcommand{\E}{\textbf{E}}

\newcommand{\Gap}[1]{\Delta_{[#1]}} 
\newcommand{\Mean}[1]{\mu_{[#1]}}   
\newcommand{\inst}{\mathcal{I}}     
\newcommand{\critical}{\mathrm{cr}} 
\newcommand{\argmax}{\operatorname*{argmax}}
\newcommand{\argmin}{\operatorname*{argmin}}
\newcommand{\pr}[1]{\Pr\left[#1\right]}
\newcommand{\OPT}{\mathrm{OPT}}     
\newcommand{\OPTc}{\overline{\OPT}}     
\newcommand{\Ind}[1]{\mathbbm{1}\left\{#1\right\}}       
\newcommand{\symdiff}{\bigtriangleup}   
\newcommand{\width}{\mathrm{width}}

\newcommand{\bestarm}{\textsc{Best-1-Arm}}
\newcommand{\topk}{\textsc{Best-$K$-Arm}}

\newcommand{\CLUCB}{\textsf{CLUCB}}
\newcommand{\Klil}{\textsf{$K$-lil'UCB}}
\newcommand{\lilCLUCB}{\textsf{lil'CLUCB}}
\newcommand{\lilLUCB}{\textsf{lil'LUCB}}
\newcommand{\lilUCB}{\textsf{lil'UCB}}
\newcommand{\LUCB}{\textsf{LUCB}}

\newcommand{\LUCBp}{\textsf{LUCB++}}
\newcommand{\LUCBr}{\textsf{lil'RandLUCB}}
\newcommand{\ExactExpGap}{\textsf{Exact-ExpGap}}

\newcommand{\eat}[1]{}

\begin{document}

\maketitle

\begin{abstract}
    In the Best-$K$ identification problem ($\topk$),
    we are given $N$ stochastic bandit arms with unknown reward distributions.
    Our goal is to identify the $K$ arms with the largest means with high confidence,
    by drawing samples from the arms adaptively.
    This problem is motivated by various practical applications
    and has attracted considerable attention in the past decade.
    In this paper, we propose new practical algorithms for the $\topk$ problem,
    which have nearly optimal sample complexity bounds
    (matching the lower bound up to logarithmic factors)
    and outperform the state-of-the-art algorithms 
    for the $\topk$ problem (even for $K=1$) in practice.
\end{abstract}

\section{Introduction}
	The stochastic multi-armed bandit models the situation where
	an agent seeks a balance between exploration and exploitation
	in the face of uncertainty in the environment.
	In the multi-armed bandit model,
	we are given $N$ bandit arms,
	each associated with a reward distribution with an unknown mean.
	Each pull of an arm results in a reward sampled independently from the corresponding reward distribution.
	The agent is then asked to meet a specific objective by pulling the arms adaptively.
	
	In the $\topk$ problem, the objective is to identify the $K$ arms with the largest means with high confidence, while minimizing the number of samples.
	$\topk$ has attracted significant attention
	and has been extensively studied in the past decade~\cite{kalyanakrishnan2010efficient,gabillon2011multi,gabillon2012best,kalyanakrishnan2012pac,bubeck2013multiple,kaufmann2013information,zhou2014optimal,kaufmann2015complexity,chen2016pure,simchowitz2017simulator,chen2017nearly}.

	On the theoretical side, sample complexity bounds of $\topk$
	have been obtained and refined in a series of recent work~\cite{gabillon2011multi,kalyanakrishnan2012pac,chen2014combinatorial,chen2016pure,chen2017nearly}.
	Recently, Chen et al.~\cite{chen2017nearly} obtained
	nearly tight bounds for $\topk$.
	They proposed an elimination-based algorithm that
	matches their new sample complexity lower bound
	(up to doubly-logarithmic factors) on any $\topk$ instance
	with Gaussian reward distributions.

	In practice, however, elimination-based algorithms tend to
	have large hidden constants in their complexity bounds,
	and may not be efficient enough
	unless the parameters are carefully tuned.%
	\footnote{
		There are typically several parameters in elimination-based algorithms.
		However, there is no principled method yet for tuning those parameters.		 
	}
	On the other hand, most existing practical $\topk$ algorithms
	(e.g., \cite{kalyanakrishnan2012pac,chen2014combinatorial})
	are based on the upper confidence bounds,
	which are originally applied to the regret-minimization multi-armed bandit problem
	by Auer et al.~\cite{auer2002finite}.
	Most of these algorithms contain an $O(H\log H)$ term in their sample complexity bounds,
	where $H = \sum_{i \in \inst}\Delta_i^{-2}$
	is the complexity measure of instance $\inst$.
	Here, we let $\Mean{j}$ denote the $j$-th largest mean among all arms, and define
		$\Delta_i = \begin{cases}
			\mu_i - \Mean{K+1}, & \mu_i\ge\Mean{K}\\
			\Mean{K} - \mu_i, & \mu_i\le\Mean{K+1}
		\end{cases}$
	as the \emph{gap} of the arm $i$ with mean $\mu_i$.
	In the worst case, these sample complexity upper bounds
	may exceed the best known theoretical bound in~\cite{chen2017nearly}
	by a factor of $\log H$, which is much larger than $\log N$
	if the gaps are small.
	Thus, it remains unclear whether there are $\topk$ algorithms
	that simultaneously achieve near-optimal sample complexity bounds
	and desirable practical performances.

	For the $\bestarm$ problem (the special case with $K = 1$),
	Jamieson et al.~\cite{jamieson2014lil} proposed the \lilUCB{} algorithm
	based on a finite form of the law of iterated logarithm (LIL, see Lemma~\ref{lem1}). 
	Later, Jamieson et al.~\cite{jamieson2014best} proposed another algorithm,
	\lilLUCB{}, by combining the LIL confidence bound and
	the \LUCB{} algorithm~\cite{kalyanakrishnan2012pac}.
	\lilLUCB{} was reported to achieve state-of-the-art performance on \bestarm{} in practice.%
	\footnote{
		Although \lilLUCB{} was proposed for \bestarm{},
		the algorithm can be easily generalized to the \topk{} problem.
	}

	In this paper, we propose two new algorithms for the $\topk$ problem, 
	$\LUCBr$ and $\lilCLUCB$, based on the LIL confidence bound
	and several previous ideas in~\cite{kalyanakrishnan2012pac,chen2014combinatorial},
	together with some additional new tricks that further improve the practical performance.%
	\footnote{
		The obtained algorithms still have provable theoretical guarantees.
	}
	We summarize our contributions as follows. 
	
	\paragraph{Algorithms for \topk{}.} We 
	propose two new \topk{} algorithms, \LUCBr{} and \lilCLUCB{}.
	Both algorithms achieve a sample complexity of
	\[
		O\left(\sum_{i\in\inst}\Delta_i^{-2}\left(\log\delta^{-1} + \log N + \log\log\Delta_i^{-1}\right)\right),
	\]
	which matches the $\Omega(H\log\delta^{-1})$ lower bound~\cite[Theorem 2]{chen2014combinatorial}
	up to $\log N$ and $\log\log\Gap{i}^{-1}$ factors.

	\paragraph{Extension to combinatorial pure exploration.}
	The \emph{combinatorial pure exploration} (CPE) problem,
	originally proposed in~\cite{chen2014combinatorial},
	generalizes the cardinality constraint in $\topk$
	(i.e., to choose a subset of arms of size exactly $K$)
	to general combinatorial structures.
	We show that the \lilCLUCB{} algorithm can be extended to the CPE problem,
	and obtain an improved sample complexity upper bound for CPE.

	\paragraph{Better practical performance.}
		We conducted extensive experiments to compare
		our new algorithms with existing \topk{} algorithms
		in terms of practical performance. 
		On various instances, \LUCBr{} takes significantly fewer samples
		than previous $\topk$ algorithms do,
		especially as the number of arms increases.
		Moreover, \LUCBr{} outperforms the state-of-the-art \lilLUCB{} algorithm~\cite{jamieson2014best} in the \bestarm{} problem.

		We remark that in a very recent work,
		Simchowitz et al.~\cite{simchowitz2017simulator},
		independently of our work,
		proposed the \LUCBp{} algorithm for \topk{} based on very similar ideas
		(see Footnote~\ref{footnote:remark} in Section~\ref{sec:lucbr} for more details). 
		Our experimental results
		indicate that our algorithm \LUCBr{} consistently outperforms \LUCBp{} in both \topk{} and \bestarm{}.

	\subsection{Related Work}
		A well-studied special case of \topk{} is the \bestarm{} problem,
		in which we are only required to identify the single arm with the largest mean.
		Although \bestarm{} has an even longer history dating back to the last century,
		understanding the exact sample complexity of \bestarm{} remains open
	     and continues to attract significant attention.
		Nearly tight sample complexity bounds as well as practical algorithms for \bestarm{}
		were obtained in a series of recent work~\cite{mannor2004sample,even2006action,audibert2010best,karnin2013almost,jamieson2014lil,jamieson2014best,chen2015optimal,carpentier2016tight,garivier2016optimal,chen2016towards}.

		In this work, we restrict our attention to algorithms that identify the \emph{exact} optimal subset. Alternatively, in the \emph{probably approximately correct learning} (PAC learning) setting, it suffices to find an approximate solution to the pure exploration problem. The sample complexity of \bestarm{}, \topk{}, and the general combinatorial pure exploration in the PAC learning setting has also been extensively studied~\cite{mannor2004sample,even2006action,kalyanakrishnan2010efficient,kalyanakrishnan2012pac,zhou2014optimal,cao2015top,chen2016pure}.

		We also focus on the problem of
		identifying the top $K$ arms with given confidence
		using as few samples as possible.
		This is known as the \emph{fixed confidence setting}
		in the bandit literature.
		In contrast, the \emph{fixed budget setting},
		initiated by \cite{audibert2010best}, models the situation that
		the number of samples (i.e., the budget) is fixed in advance,
		while the objective is to minimize the risk of outputing an incorrect answer.
		The fixed budget setting of the pure exploration problem in multi-armed bandits
		have also been thoroughly studied recently~\cite{gabillon2011multi,bubeck2013multiple,bubeck2012regret,gabillon2012best,karnin2013almost,chen2014combinatorial,kaufmann2015complexity,carpentier2016tight,gabillon2016improved}.

\section{Preliminaries}
	\paragraph{Notations.} For an instance $\inst$ of the $\topk$ problem, $\mu_i$ denotes the mean of arm with index $i$, while $\Mean{i}$ denotes the $i$-th largest mean in $\inst$. The gap of arm $i$ is defined as
	\[
		\Delta_i = \begin{cases}
			\mu_i - \Mean{K+1}, & \mu_i\ge\Mean{K},\\
			\Mean{K} - \mu_i, & \mu_i\le\Mean{K+1}\text{,}
		\end{cases}
	\]
	and $\Gap{i}$ denotes the gap of the arm with the $i$-th highest mean.
	$H=\sum_{i\in\inst}\Delta_i^{-2}$ is the complexity measure of the instance $\inst$.
	For a specific instance,
	we let $\OPT$ denote the set of the $K$ arms with the largest means,
	while $\OPTc$ stands for $\inst\setminus\OPT$.

	Throughout this paper, $\log$ stands for the natural logarithm. $\Ind{A}$ denotes the indicator random variable for event $A$.

	\paragraph{Finite Form of LIL.}
	The following lemma is a finite form of the law of iterated logarithm~\cite[Lemma~3]{jamieson2014lil}.

	\begin{lemma}[Finite Form of LIL]\label{lem1}
	\label{lem: LIL}
		Let $X_1,X_2,\ldots $ be i.i.d. centered $\sigma$-sub-Gaussian random variables,
		i.e., for any $t\in\mathbb{R}$ and $i\in\mathbb{N}$,
			$\E\left[e^{tX_i}\right]\leq e^{\sigma^2t^2/2}$.
		Then for any $\epsilon\in (0,1)$ and $\delta\in (0,\log (1+\epsilon)/e)$,
		with probability at least $1-c_\epsilon\delta^{1+\epsilon}$,
			\[\frac{1}{t}\sum_{s=1}^t X_s \leq U(t,\delta)\]
		holds for all $t \ge 1$.
		Here, 
		\[
			U(t,\omega)=(1+\sqrt{\epsilon})\sqrt{\frac{2\sigma^2(1+\epsilon)}{t}\log\left[\frac{\log ((1+\epsilon)t)}{\omega}\right]}
		\]
		and
		\[
			c_\epsilon=\frac{2+\epsilon}{\epsilon}\left[\frac{1}{\log(1+\epsilon)}\right]^{1+\epsilon}\text{.}
		\]
	\end{lemma}
	For brevity, we adopt the notations $U(t,\omega)$ and $c_\epsilon$ throughout the paper.

\section{\LUCBr{}}\label{sec:lucbr}
\LUCBr{} is inspired by the \LUCB{} algorithm~\cite{kalyanakrishnan2012pac}
and the LIL confidence bound~\cite{jamieson2014lil}.
The algorithm uses a time-independent confidence radius based on LIL (Lemma~\ref{lem1}),
whereas the confidence radius of \LUCB{} is time-dependent.
\LUCBr{} also draws samples more efficiently than \LUCB{}
by taking advantage of a novel randomized sampling rule.

\LUCBr{} starts by sampling each arm once
and proceeds similarly as \LUCB{} does.
Let $T_i(t)$ denote the number of samples drawn from arm $i$ up to time $t$,
and let $\muhat_{i,T_i(t)}$ be the empirical mean of arm $i$ at time $t$.
In each iteration, \LUCBr{} computes the confidence radius for each arm:
inspired by \LUCBp{}, we choose a confidence radius of
    $u_{i,T_i(t)} = U(T_i(t),\delta/2(N-K))$
for each arm in $\High_t$,
the set of the $K$ arms with the highest empirical means at time $t$.
A confidence radius of
    $u_{i,T_i(t)} = U(T_i(t),\delta/2K)$
is used for arms in $\Low_t = \inst \setminus \High_t$.%
\footnote{\label{footnote:remark}
    In a preliminary version of \LUCBr{}, 
    we set the confidence radius of each arm $i$ to be 
    $U(T_i(t),\delta/N)$.
    Our current modification is inspired by the work of~\cite{simchowitz2017simulator}.
    In fact, their algorithm \LUCBp{}
    is based on a very similar idea of combining \LUCB{} and LIL confidence bound
    without the randomized sampling rule.
    Note that the modification slightly improves the practical performance in 
    certain cases, but our theoretical upper bound and all proofs remain the same.
}
The upper confidence bound (UCB) for each arm $i \in \Low_t$ is then computed as
    $\empmean{i}{T_i(t)} + u_{i, T_i(t)}$,
while the lower confidence bound (LCB) of arm $i \in \High_t$ is given by
    $\empmean{i}{T_i(t)} - u_{i, T_i(t)}$.

The algorithm terminates and outputs $\High_t$
if the lowest LCB in $\High_t$ (achieved by arm $h_t$)
is greater than or equal to the highest UCB in $\Low_t$ (achieved by $l_t$).
Otherwise, we apply the following sampling rule:
instead of sampling both $h_t$ and $l_t$ as in \LUCB{} and \LUCBp{},
we sample $h_t$ with probability $T_{l_t}(t)/(T_{h_t}(t)+T_{l_t}(t))$,
and sample $l_t$ otherwise.
In other words, we tend to pull the arm from which we have drawn fewer samples.%
\footnote{
    The idea of pulling the arm from which fewer samples have been drawn
    also appears in the previous work of Gabillon et al.~\cite{gabillon2012best}.
    Our preliminary experiments show that their deterministic sampling rule
    is less efficient in practice.
}

The \LUCBr{} algorithm is formally defined in Algorithm~\ref{alg:Rand-LUCB}.

\begin{algorithm2e}[h]
\caption{$\LUCBr(\inst,K,\delta)$}
\label{alg:Rand-LUCB}
Initialization: Sample each arm once, and let $T_i(N)\gets 1$ for each $i\in\inst$\;
\For{$t=N,N+1,...$} {
    $\High_t$ $\gets$ $K$ arms with the highest empirical means\;
    $\Low_t \gets \inst \setminus \High_t$\;
    \lFor{$i \in \High_t$} {
        $u_{i,T_i(t)} \gets U\left(T_i(t), \delta/[2(N-K)]\right)$
    }
    \lFor{$i \in \Low_t$} {
        $u_{i,T_i(t)} \gets U\left(T_i(t), \delta/(2K)\right)$
    }
    $h_t\gets\argmin_{i\in\textrm{High}_t}\empmean{i}{T_i(t)}-u_{i,T_i(t)}$\;
    $l_t\gets\argmax_{i\in\textrm{Low}_t}\empmean{i}{T_i(t)}+u_{i,T_i(t)}$\;
    \lIf{$\muhat_{h_t,T_{h_t}(t)}-u_{h_t,T_{h_t}(t)}\geq \muhat_{l_t,T_{l_t}(t)}+u_{l_t,T_{l_t}(t)}$} {
        \Return $\High_t$
    }
    With probability $T_{l_t}(t)/(T_{h_t}(t)+T_{l_t}(t))$, sample $h_t$;
    otherwise, sample $l_t$\; 
    \lFor{$i\in\inst$}{$T_i(t+1)\gets T_i(t) + \Ind{i\text{ is sampled}}$}
}
\end{algorithm2e}

The following theorem, whose proof appears in Appendix~\ref{app:LUCBr}, states the performance guarantees of the \LUCBr{} algorithm.
\begin{theorem}[\LUCBr{}]\label{theo:LUCBr}
With probability at least $1-\ceps\delta$, where $\ceps$ is defined in Lemma~\ref{lem: LIL}, \LUCBr{} outputs the correct answer and takes
\[
O\left(\sum_{i\in\inst}\Delta_i^{-2}\left(\log\delta^{-1}+\log N+\log\log\Delta_i^{-1}\right)\right)
\]
samples in expectation.
\end{theorem}

\section{\lilCLUCB{}}
    \subsection{Algorithm for \topk{}}
        In this section, we combine LIL (Lemma~\ref{lem1})
        with the \CLUCB{} algorithm~\cite{chen2014combinatorial} to obtain \lilCLUCB{},
        which achieves the same sample complexity bound as \LUCBr{} does in \topk{},
        and can be generalized to the combinatorial pure exploration problem
        (see Section~\ref{sec:lilCLUCB-gen}).
 
        In \lilCLUCB{}, the confidence radius of an arm $i$ at time $t$
        is set to $U(T_i(t),\delta/N)$.
        The algorithm starts by sampling each arm once.
        In each round, let $M_t$ be the set of $K$ arms with the highest empirical mean.
        We define the ``revised'' mean of each arm in $M_t$ as
        its empirical mean minus its confidence radius (i.e., its lower confidence bound);
        the revised mean of each arm not in $M_t$ is its upper confidence bound.
        Let $\widetilde{M}_t$ be the set of the $K$ arms with the highest revised mean.
        The algorithm terminates and outputs $M_t$ if $M_t = \widetilde{M}_t$.
        Otherwise, the algorithm samples the arm with the largest confidence radius in
            $M_t\symdiff\widetilde{M}_t$,
        the symmetric difference between $M_t$ and $\widetilde{M}_t$.
        A formal description of \lilCLUCB{} is shown in Algorithm~\ref{alg:lil-CLUCB} in Appendix~\ref{alg: lil'CLUCB}.
        The theoretical guarantee of \lilCLUCB{} is stated in Theorem~\ref{theo:CLUCB},
        whose proof is deferred to Appendix~\ref{msproof: CLUCB}.
        \begin{theorem}[\lilCLUCB{}]\label{theo:CLUCB}
            For any $\delta\in(0,1)$, with probability at least
                $1-c_\epsilon\delta^{1+\epsilon}/N^\epsilon$,
            \lilCLUCB{} returns the correct answer and takes at most
            \[
            O\left(\sum_{i\in \inst}\Delta_i^{-2}\left(\log\delta^{-1}+\log N+\log\log\Delta_i^{-1}\right)\right)
            \]
            samples.
        \end{theorem}

    \subsection{Generalization to Combinatorial Pure Exploration}\label{sec:lilCLUCB-gen}
        In the CPE problem, we are given a decision class $\mathcal{M}\subseteq 2^{[N]}$,
        which is a collection of feasible subsets of arms.
        The goal is to identify the set $\OPT=\argmax_{M\in\mathcal{M}}\mu(M)$,
        where $\mu(M)=\sum_{i \in M} \mu_i$ is the total mean of subset $M$.
        We allow the algorithm to access a \emph{maximization oracle}, which, on input $\mathbf{v}$,
        returns $\argmax_{M\in \mathcal{M}}\sum_{i \in M}\mathbf{v}(i)$.

        The \lilCLUCB{} algorithm can be generalized to the combinatorial pure exploration (CPE) problem.
        The generalized version of \lilCLUCB{} has a better theoretical guarantee than the \CLUCB{} algorithm~\cite{chen2014combinatorial}.
        The generalized \lilCLUCB{} algorithm is similar to \lilCLUCB{} except that
        the algorithm calls the maximization oracle
        with parameters $\muhat$ and $\mutil$ to obtain $M_t$ and $\widetilde{M}_t$.
        A more detailed description of the algorithm appears in Appendix~\ref{lil'CLUCB:generalization}.

        The hardness of an instance of CPE is captured by the complexity measure $H=\sum_{i\in \inst}\Delta_i^{-2}$, where the gap $\Delta_i$ in CPE is defined as
            \begin{equation*}
                \Delta_i = \begin{cases}
                    \mu(\OPT) - \max_{M\in\mathcal{M}, i \notin M}\mu(M),   & i\in\OPT,\\
                    \mu(\OPT) - \max_{M\in\mathcal{M}, i \in M}\mu(M),      & i\notin\OPT.
                \end{cases}
            \end{equation*}

        Another important quantity for the analysis of sample complexity in CPE is the \emph{width} of the decision class $\mathcal{M}$, denoted by $\width(\mathcal{M})$~\cite[Section~3.1]{chen2014combinatorial}. For instance, the width of matroid (which includes \topk{} as a special case) is at most 2. The next theorem, proved in Appendix~\ref{lil'CLUCB:generalization}, states the sample complexity of the generalized \lilCLUCB{} algorithm.

        \begin{theorem}[Generalized \lilCLUCB{}]
        \label{theo:GCLUCB}
        Suppose the reward distribution of each arm is $\sigma$-sub-Gaussian.
        For any $\delta\in(0,1)$ and
        decision class $\mathcal{M}\subseteq 2^{[N]}$,
        with probability at least $1-c_\epsilon\delta^{1+\epsilon}/N^\epsilon$,
        \lilCLUCB{} outputs the correct answer and takes at most
        \[
            O\left(\width(\mathcal{M})^2\sigma^2\sum_{i\in \inst}\Delta_i^{-2}\left(\log\delta^{-1}+\log N+\log\log\Delta_i^{-1}\right)\right)
        \]
        samples.
        \end{theorem}

        This improves the previous 
            $O\Bigg(\width(\mathcal{M})^2\sigma^2 H \log(N H/\delta)\Bigg)$
        upper bound in \cite{chen2014combinatorial}.
        Note that the $H\log H$ term in the previous bound
        can be much larger than the $\sum_{i\in \inst}\Delta_i^{-2}\log\log\Delta_i^{-1}$ term
        in our sample complexity bound.

\section{Experiments}\label{exp}
In this section, we investigate the practical performance of our new algorithms for \topk{}
and compare them to the state-of-the-art algorithms.
In all our experiments, the rewards are Gaussian random variables with variance $1 / 4$.%
\footnote{This is a standard setup in the literature (e.g., \cite{jamieson2014lil}).}
The confidence level, i.e., the probability of making a mistake, is set to $\nu = 0.01$.
The definition of $U(t,\omega)$ in Lemma~\ref{lem1} is also slightly modified to
	\[(1+\sqrt{\epsilon})\sqrt{\frac{2\sigma^2(1+\epsilon)}{t}\log\left[\frac{\log ((1+\epsilon)t+2)}{\omega}\right]},\]
in order to avoid negative numbers inside the $\log$.


\subsection{\topk{} Experiments}\label{sec:expr-top-k}
	\paragraph{Setup.}
	We evaluate the \topk{} algorithms on the following two sets of instances:
	\begin{itemize}
		\item ``1-sparse'' instances: $\Mean{1}=\cdots=\Mean{K}=1/2$ and $\Mean{K+1}=\cdots=\Mean{N}=0$. $K$ is set to either $2$ or $N / 2$.
		\item ``$\alpha$-Exponential'' instances:
			$\Mean{i} = \frac{N-K}{N}+\frac{K}{N}\left(\frac{K-i}{K}\right)^{\alpha}$
		for $i\leq K$ and 
			$\Mean{i} = \frac{N-K}{N}-\frac{N-K}{N}\left(\frac{i-K}{N-K}\right)^{\alpha}$
		for $i>K$. Here $\alpha = 0.3$, and $K$ is set to either $2$ or $N / 2$.
	\end{itemize}

	We test the following five algorithms:
	\LUCB{}~\cite{kalyanakrishnan2012pac},
	\lilLUCB{}~\cite{jamieson2014best},
	\LUCBp{}~\cite{simchowitz2017simulator},
	\LUCBr{} (Algorithm~\ref{alg:Rand-LUCB}) and
	\lilCLUCB{} (Algorithm~\ref{alg:lil-CLUCB}).
	For the algorithms based on the LIL confidence bounds,
	we set the parameters $\epsilon$ and $\delta$ such that the error probability
	 (see Theorem~\ref{theo:LUCBr})
	is smaller than the required confidence level $\nu = 0.01$.

	Besides the five algorithms above with provable performance guarantees
	(which we call the ``faithful versions''),
	we also evaluate the ``heuristic versions'' of these algorithms (except \LUCB{}).
	In particular, similar to previous experiments by Jamieson et al.~\cite{jamieson2014lil},
	we set the parameters to $\epsilon = 0$ and $\delta = \nu$.
	These parameters may not satisfy the theoretical conditions for the performance guarantee,
	but are sufficient for any practical usage.%
	\footnote{
		In fact, during the course of our experiment,
		we have not seen a single case where any one of these algorithms fail to
		output the correct answer. This suggests that the theoretical analysis is quite pessimistic.
	}

	\paragraph{Results.}
	The results are shown in Figure~\ref{fig:Top-K}.
	On instances with $K=2$, \LUCBr{} draws significantly fewer samples
	than the other four algorithms do, especially on instances with many arms.
	When $K=N/2$, \LUCBr{} outperforms \LUCB{} and \lilCLUCB{},
	and has a similar performance to \LUCBp{} (which coincides with \lilLUCB{} when $K = N / 2$).
	This indicates that the randomized sampling rule in \LUCBr{}
	is more efficient in practice, especially in the case when $K$ is small.
	Moreover, it is more practical than the sampling rule of \lilCLUCB.

	Our explanation is that in the asymmetric case where $K$ is small,
	if both marginal arms (i.e. $h_t$ and $l_t$ in \LUCBr{}) are sampled,
	we would take significantly more samples from each of the top $K$ arms
	than from each of the other $N - K$ arms.
	The randomized sampling rule succeeds in balancing the number of samples from
	each of those two groups of arms.
	The stopping condition is thus reached earlier when the randomized sampling rule is adopted.
	However, this idea of balancing the samples is overdone
	by the deterministic sampling rule of Gabillon et al.~\cite{gabillon2012best},
	which enforces the same number of draws from each of the marginal arms.
	As a result, the deterministic sampling rule is found to be no more efficient
	than pulling both marginal arms.
	In the symmetric case where $K=N/2$, however,
	the number of samples from each of the top $K$ arms and each of the other $N - K$ arms are roughly the same,
	so the overall effect of the random sampling rule is similar to sampling from both marginal arms.

		 
\subsection{\bestarm{} Experiments}\label{sec:expr-top-1}
	\paragraph{Setup.}
	For the special case of \bestarm{},
	we evaluate the five algorithms tested in Section~\ref{sec:expr-top-k} and the \lilUCB{} algorithm~\cite{jamieson2014lil}.
	All these algorithms are tested on the following three sets of instances in~\cite{jamieson2014lil}:
	\begin{itemize}
		\item ``1-sparse'' instances:
			$\Mean{1} = 1 / 2$ and $\Mean{i} = 0$ for $i > 1$.
		\item ``lil-Exponential'' instances:
			$\Mean{1} = 1$ and $\Mean{i} = 1-((i-1)/N)^{\alpha}$ for $i > 1$.
			We set $\alpha$ to either $0.3$ or $0.6$. 
	\end{itemize}

	\paragraph{Results.}
	The results are shown in Figure~\ref{fig:Top-1}.
	We note that the performance of the faithful version of \lilUCB{}
	is worse than those of other algorithms by an order of magnitude.
	For clearness, we exclude the presentation of it from the charts.
	We note that \LUCBr{} significantly outperforms all the other algorithms
	on the ``1-sparse'' and ``$\alpha=0.3$'' instances, especially when the number of arms gets larger.
	In the ``$\alpha=0.6$'' scenario, \LUCBr{} and \LUCBp{} have similar performances. 		      
	All these algorithms take much fewer samples than \LUCB{} and \lilCLUCB{} do, and have much better performance than \lilLUCB{} when the heuristic verions are considered.

	\begin{figure*}[htp!]
	\centering
	\subfigure[(Faithful) 1-Sparse, $K=2$]{
		\includegraphics[width = 2.65in,height=1.65in]{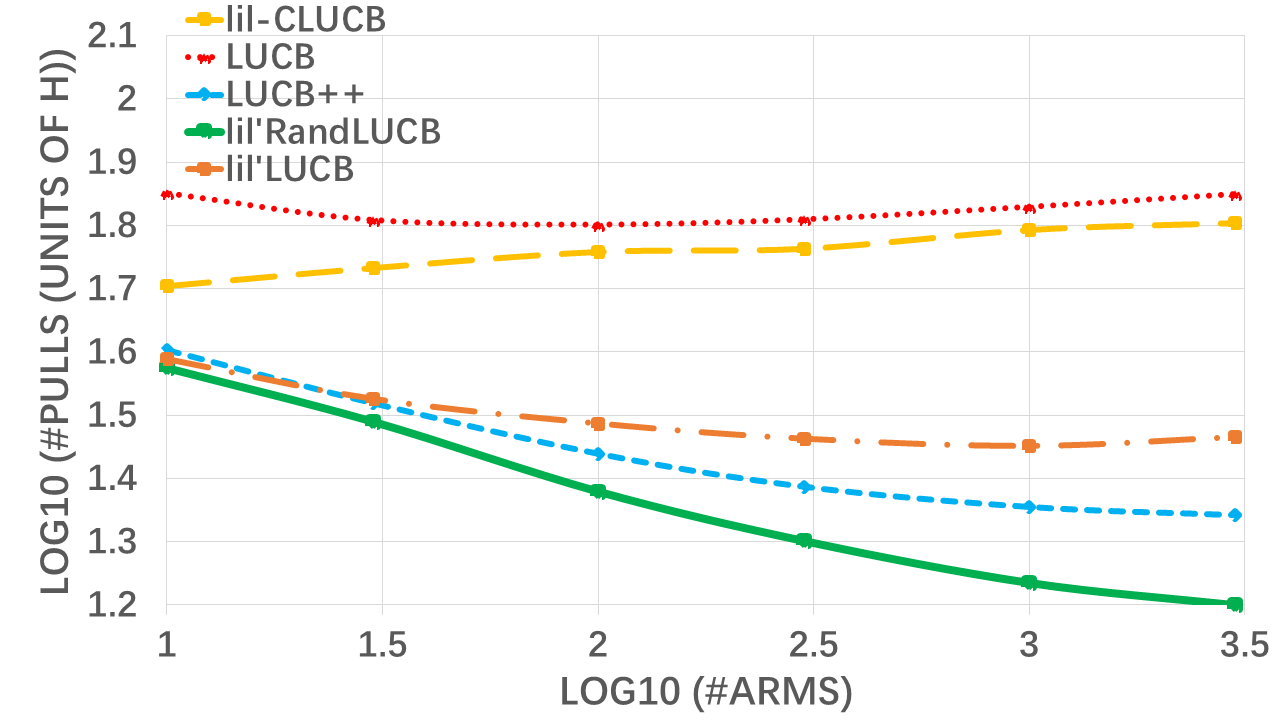}
	}
	\subfigure[(Heuristic) 1-Sparse, $K=2$]{
		\includegraphics[width = 2.65in,height=1.65in]{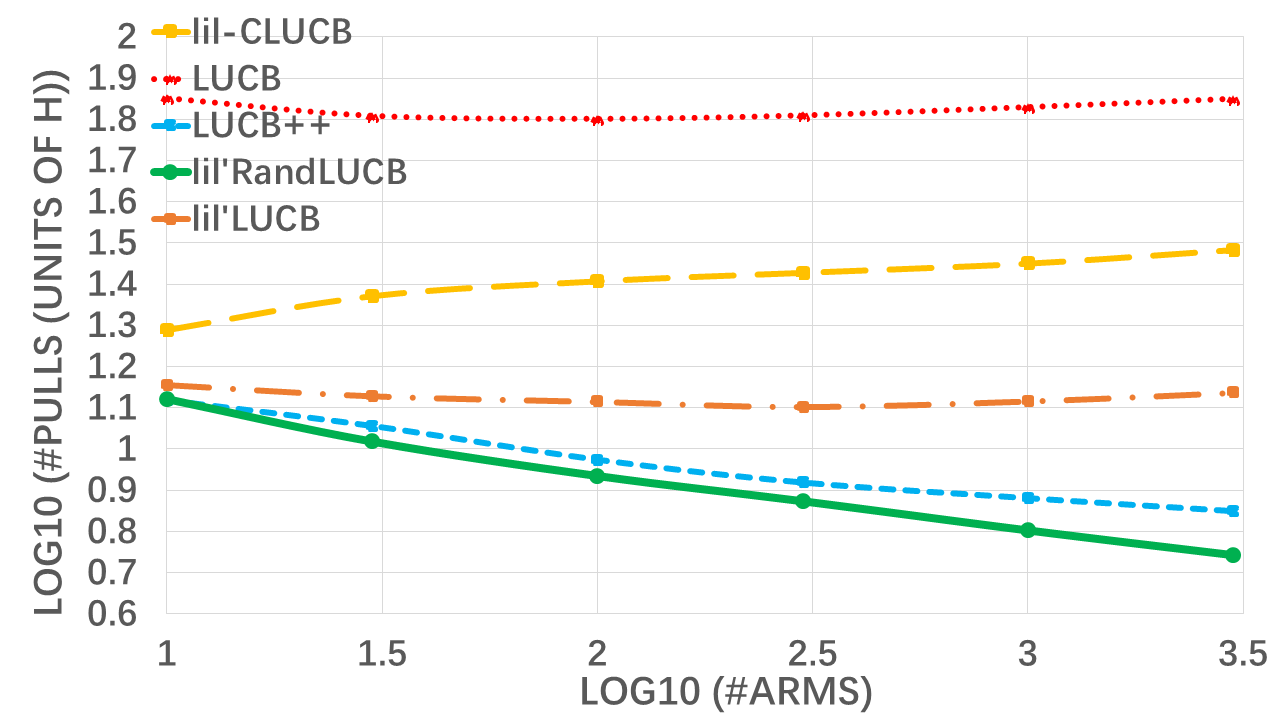}
	}
	\\
	\subfigure[(Faithful) 1-Sparse, $K=N/2$]{
		\includegraphics[width = 2.65in,height=1.65in]{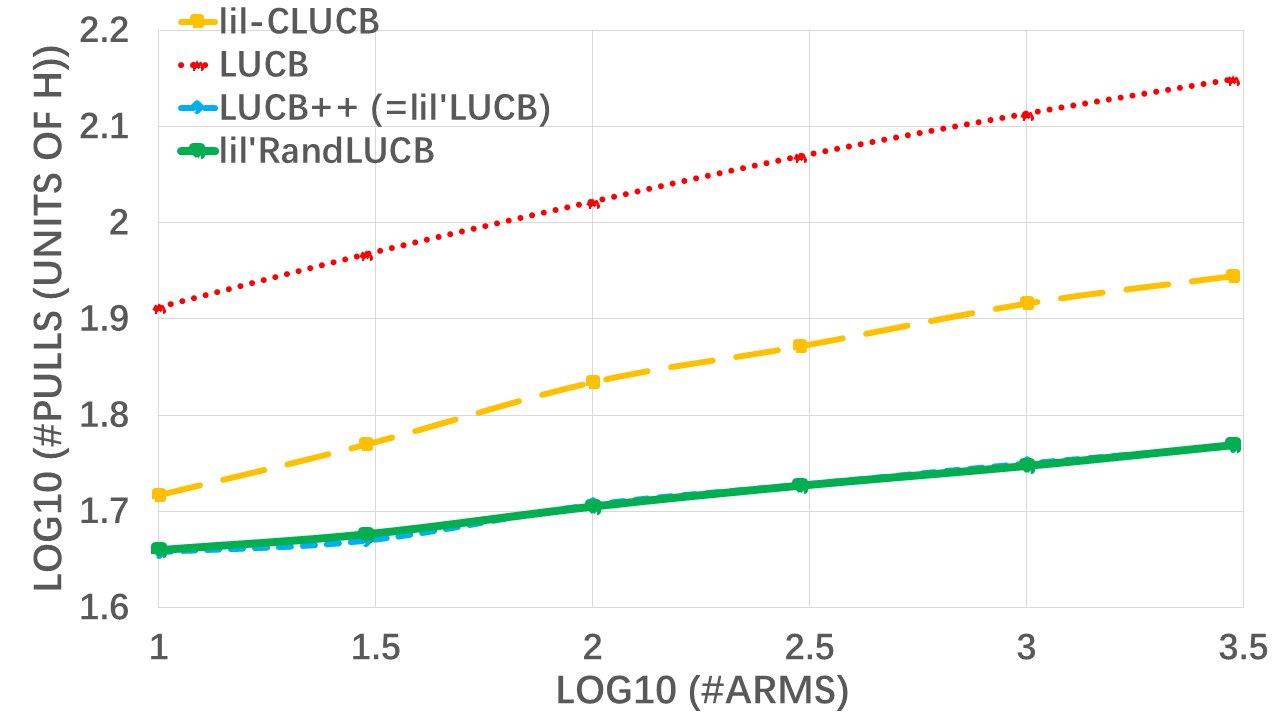}
	}
	\subfigure[(Heuristic) 1-Sparse, $K=N/2$]{
		\includegraphics[width = 2.65in,height=1.65in]{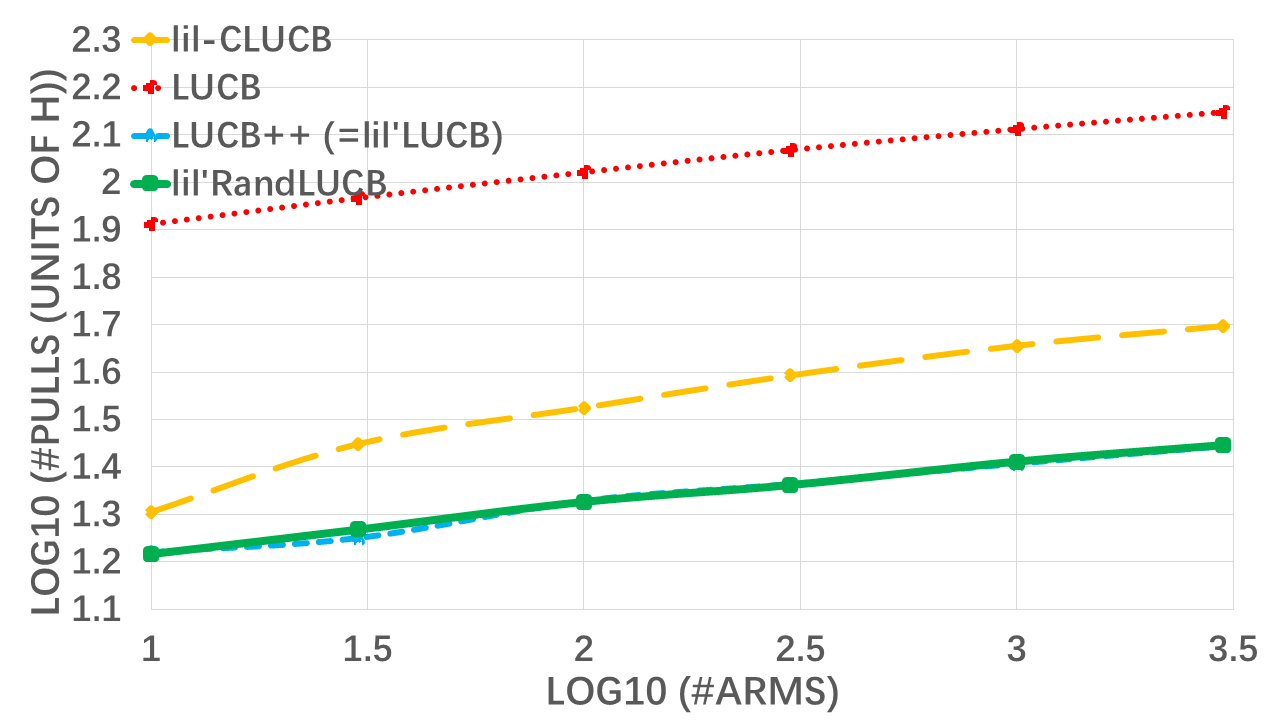}
	}
	\\
	\subfigure[(Faithful) 0.3-Exponential, $K=2$]{
		\includegraphics[width = 2.65in,height=1.65in]{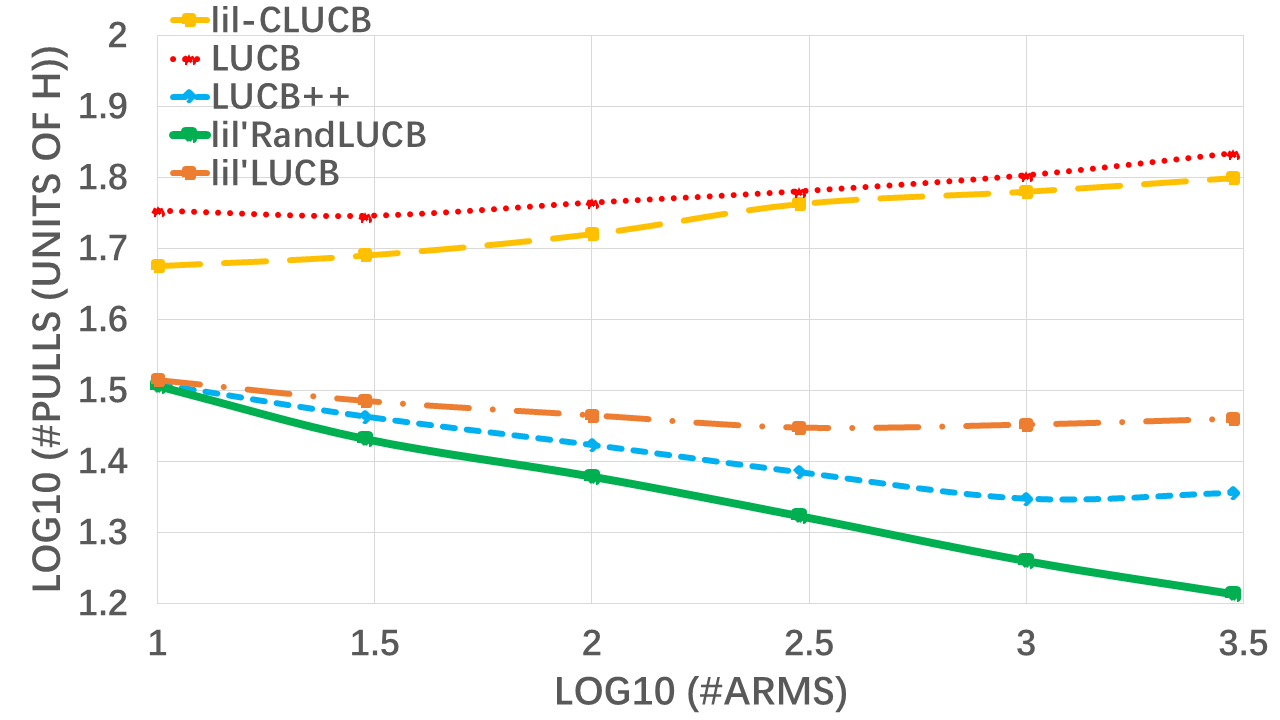}
	}
	\subfigure[(Heuristic) 0.3-Exponential, $K=2$]{
		\includegraphics[width = 2.65in,height=1.65in]{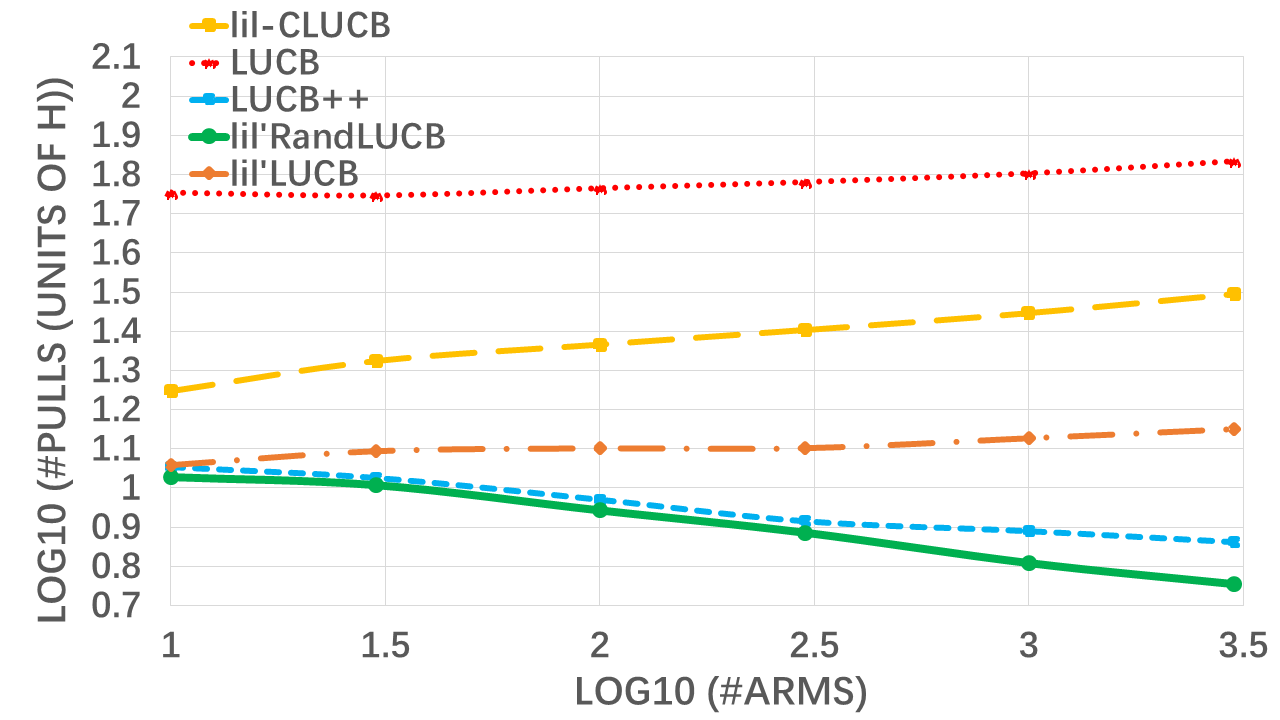}
	}
	\\
	\subfigure[(Faithful) 0.3-Exponential, $K=N/2$]{
		\includegraphics[width = 2.65in,height=1.65in]{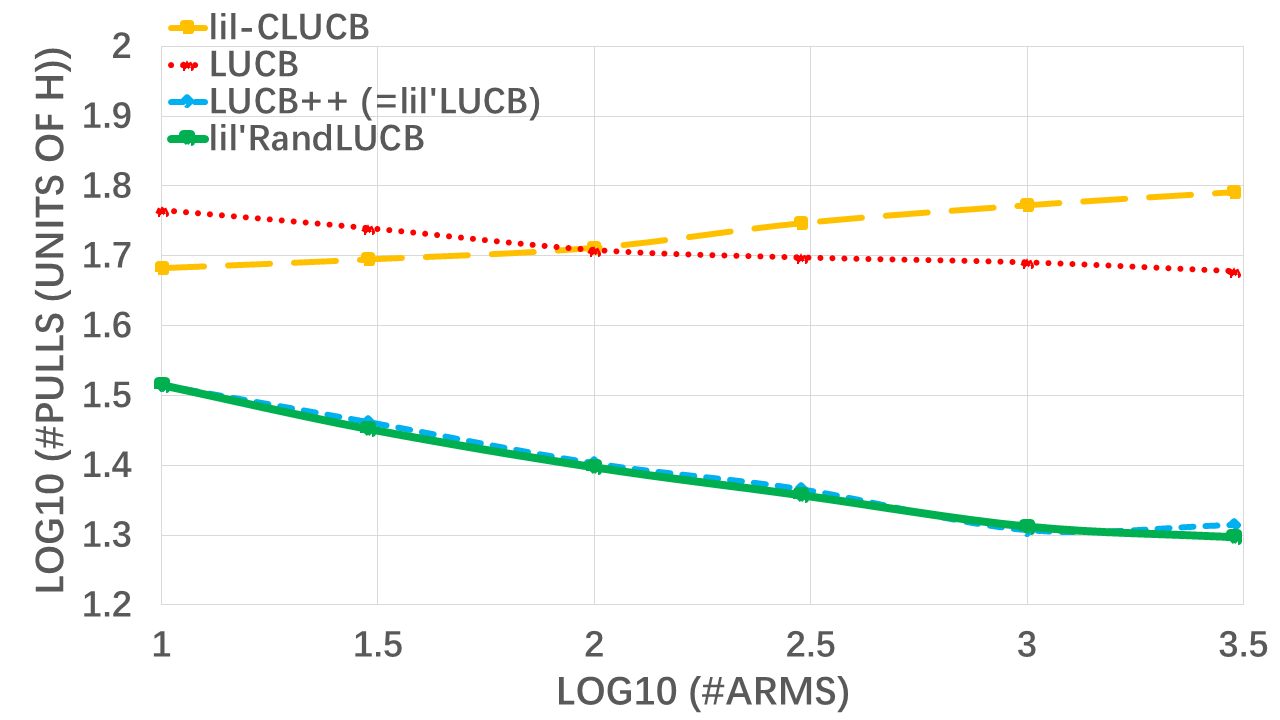}
	}
	\subfigure[(Heuristic) 0.3-Exponential, $K=N/2$]{
		\includegraphics[width = 2.65in,height=1.65in]{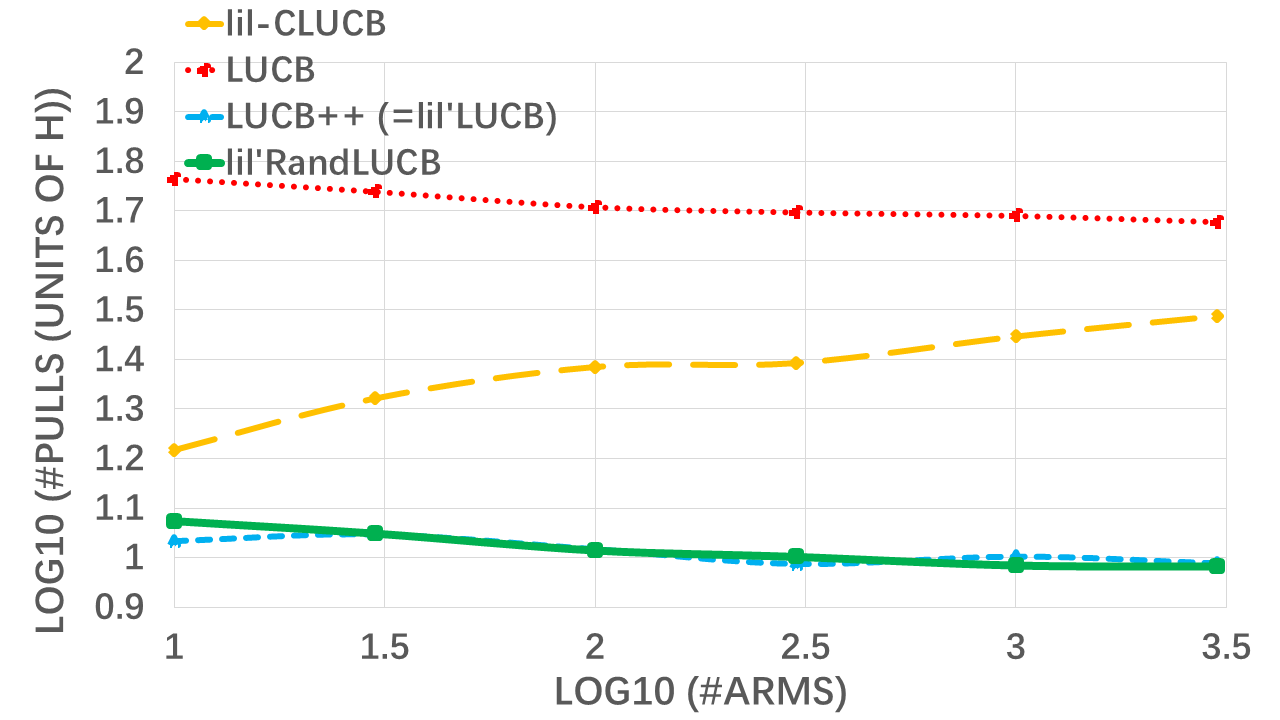}
	}
	\caption{Results for \topk{}, where $K=2$ or $K=N/2$.}
	\label{fig:Top-K}
	\end{figure*}

	\begin{figure}[htp!]
	\subfigure[(Faithful) 1-Sparse, K=1]{
		\includegraphics[width = 2.7in,height=1.65in]{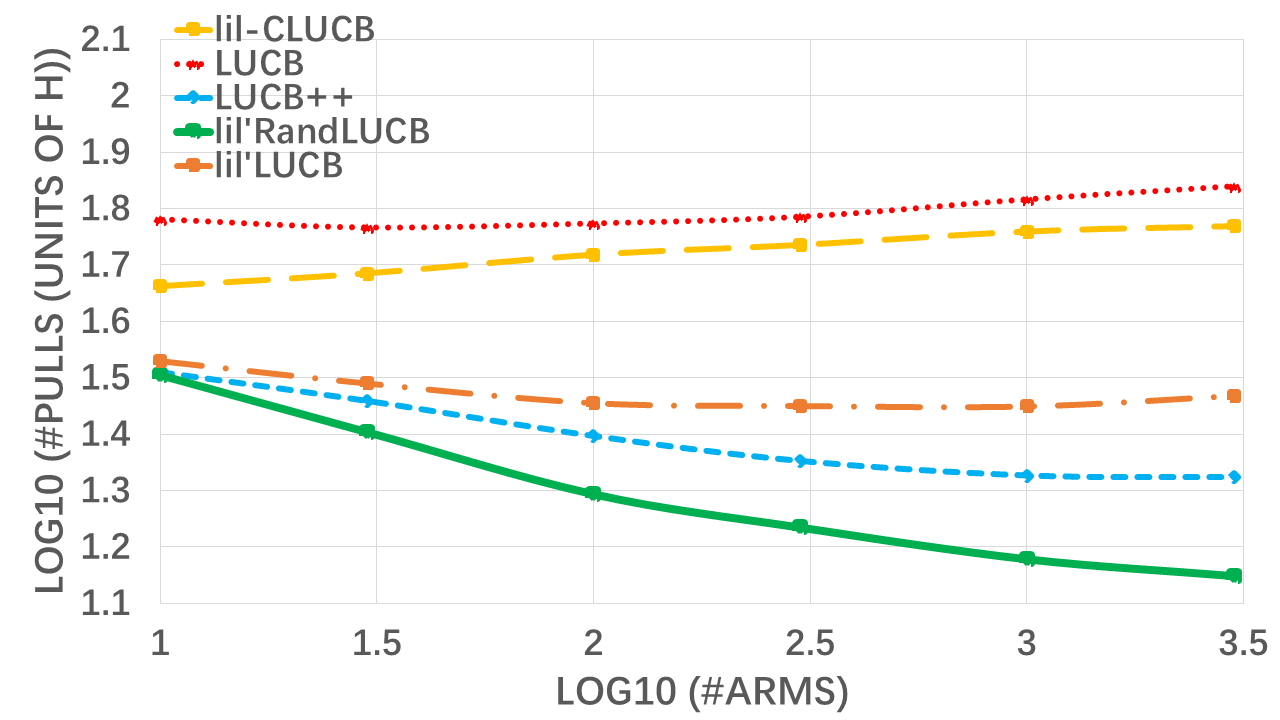}
	}
	\subfigure[(Heuristic) 1-Sparse, K=1]{
		\includegraphics[width = 2.7in,height=1.65in]{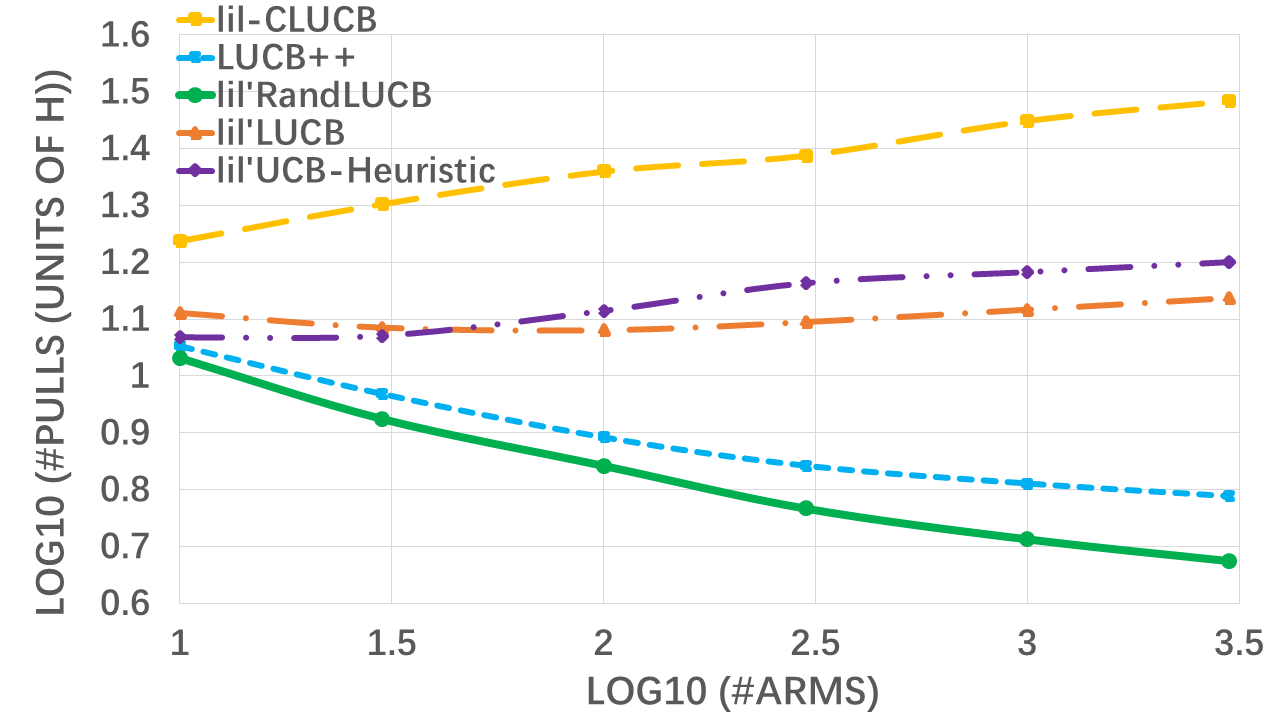}
	}
	\\
	\subfigure[(Faithful) $\alpha=0.3$, K=1]{
		\includegraphics[width=2.7in,height=1.65in]{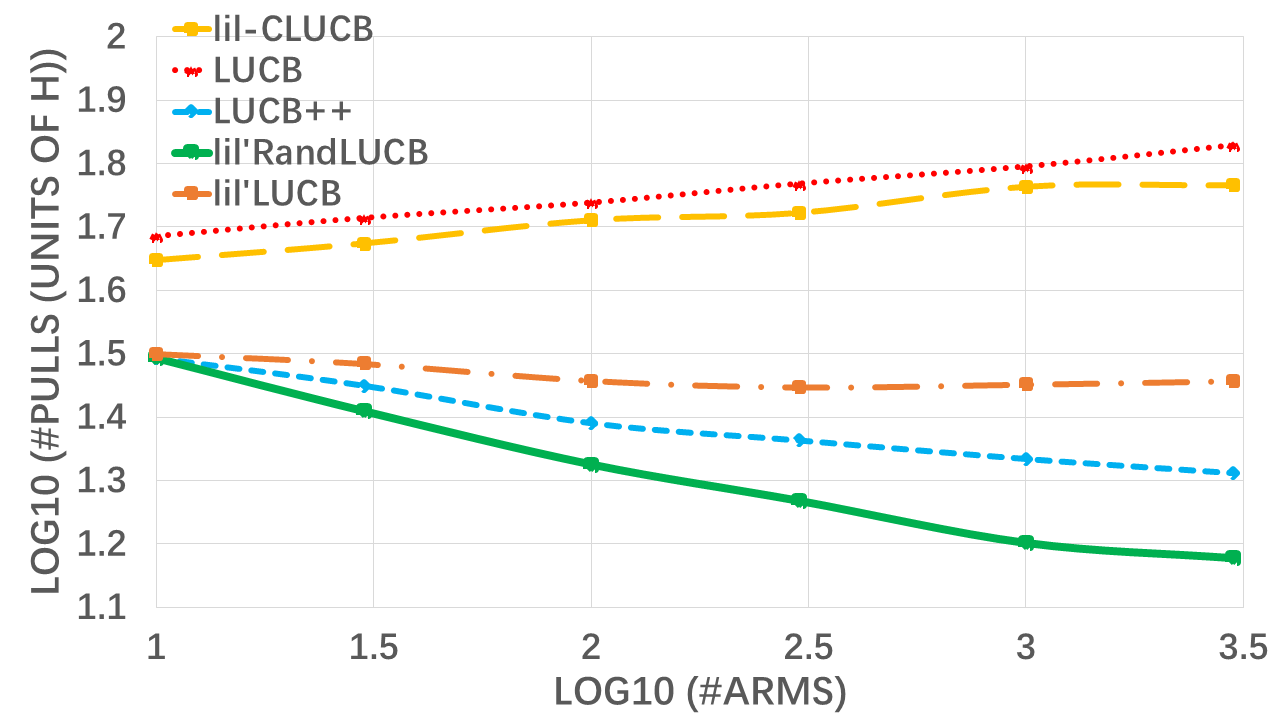}
	}
	\subfigure[(Heuristic) $\alpha=0.3$, K=1]{
		\includegraphics[width=2.7in,height=1.65in]{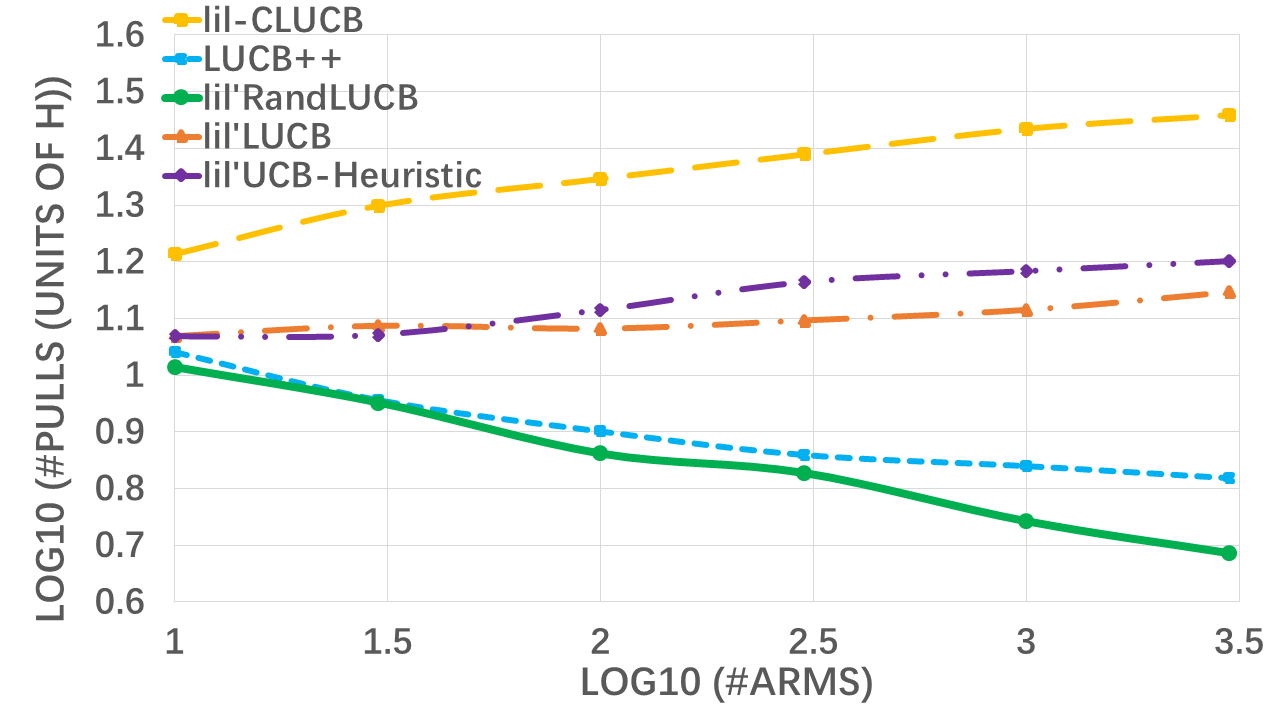}
	}
	\\
	\subfigure[(Faithful) $\alpha=0.6$, K=1]{
		\includegraphics[width=2.7in,height=1.65in]{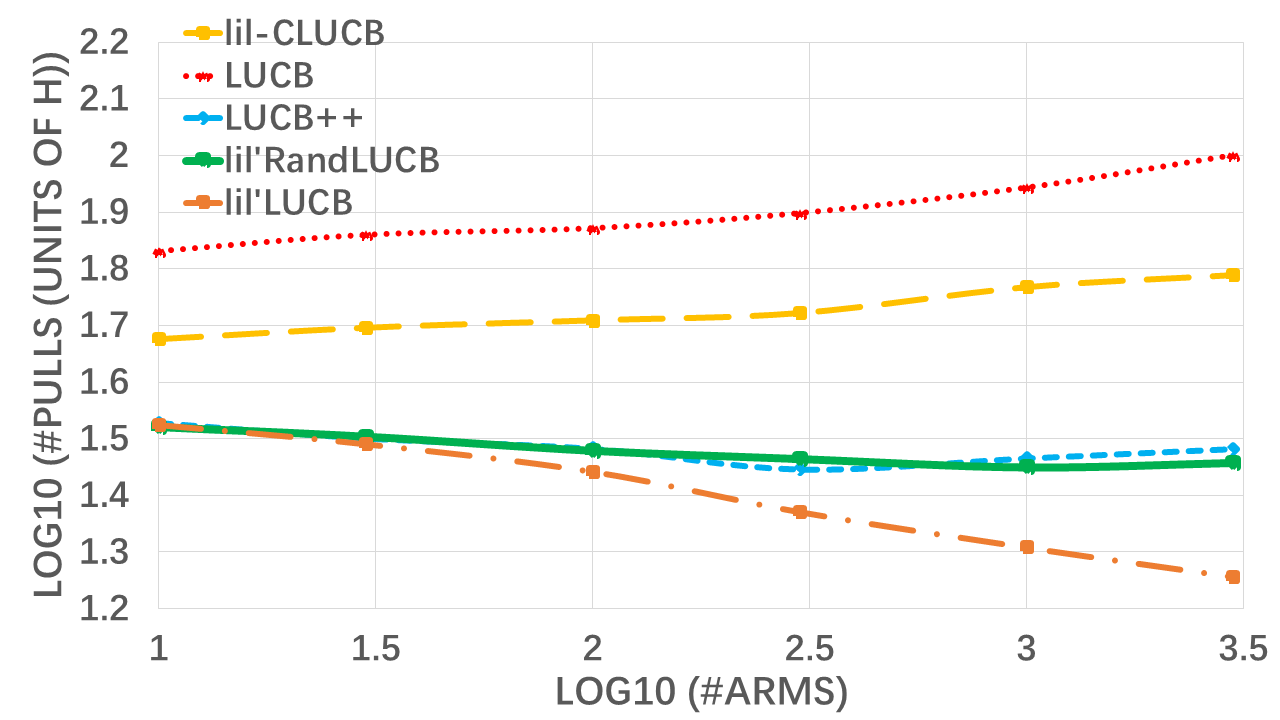}
	}
	\subfigure[(Heuristic) $\alpha=0.6$, K=1]{
		\includegraphics[width=2.7in,height=1.65in]{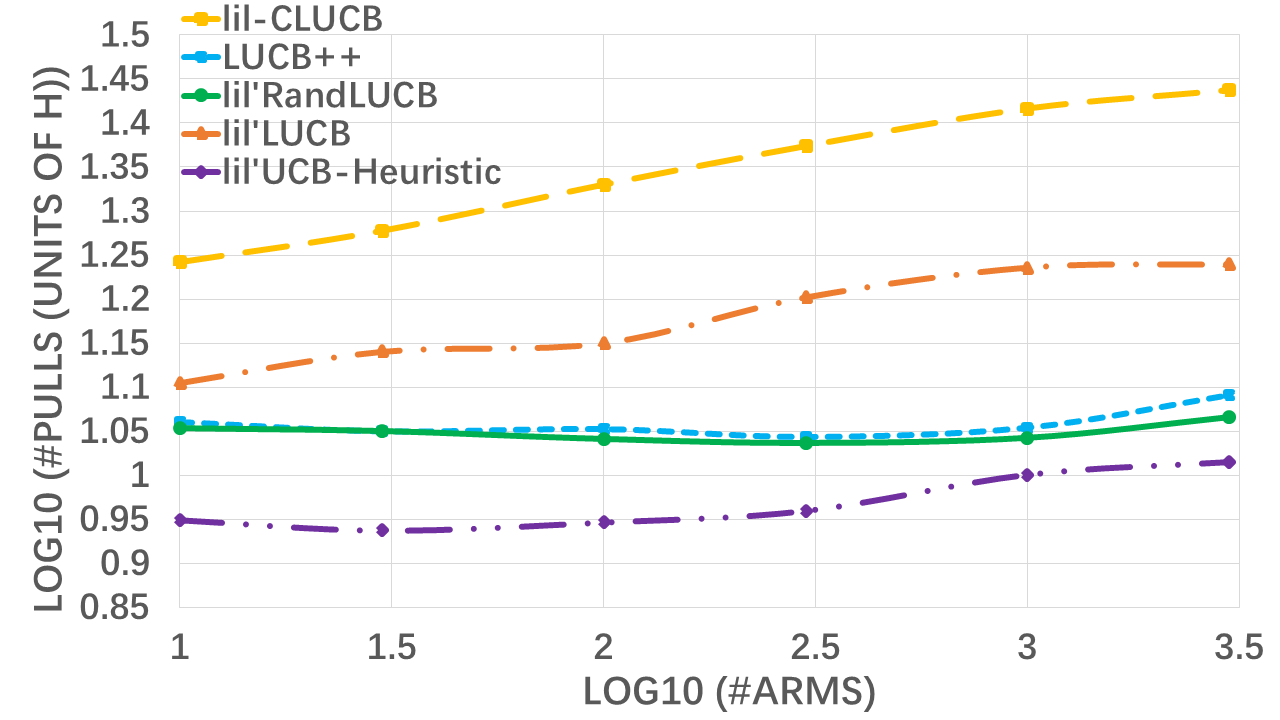}
	}
	\caption{Results for \bestarm{}.}
	\label{fig:Top-1}
	\end{figure}

\bibliography{topk}
\bibliographystyle{alpha}

\newpage

\appendix
\section{A Technical Lemma}
The following lemma follows from a direct calculation.
\begin{lemma}\label{lem2}
For $t\geq 1,\epsilon\in (0,1), c>0,0<\omega\leq 1$,
\[
\frac{1}{t}\log\Bigg(\frac{\log((1+\epsilon)t)}{\omega}\Bigg)\geq c 
\implies
t\leq \frac{1}{c}\log\Bigg(\frac{2\log((1+\epsilon)/(c\omega))}{\omega}\large\Bigg).
\]
For $t\geq 1,s\geq 3, \epsilon\in(0,1],0<\omega\leq\delta\leq e^{-e}$,
\[
\frac{1}{t}\log\Bigg(\frac{\log((1+\epsilon)t)}{\omega}\Bigg)\geq \frac{c}{s}\log\Bigg(\frac{\log((1+\epsilon)s)}{\delta}\Bigg)
\implies
t\leq\frac{s}{c}\cdot\frac{\log(2\log(\frac{1}{cw})/\omega)}{\log(1/\delta)}.
\]
\end{lemma}

\section{Missing Proofs for \LUCBr{}}\label{app:LUCBr}
In this section we present the missing proof of Theorem~\ref{theo:LUCBr}, which we restate below for convenience.

\noindent\textbf{Theorem~\ref{theo:LUCBr}} (restated) \textit{
With probability at least $1-\ceps\delta$, \LUCBr{} outputs the correct answer and takes
\[
O\left(\sum_{i\in\inst}\Delta_i^{-2}\left(\log\delta^{-1}+\log N+\log\log\Delta_i^{-1}\right)\right)
\]
samples in expectation.
}

We start with the following simple lemma, which relates the difference between the means of two arms to their gaps.
\begin{lemma}\label{lem:mean-vs-gap}
    For any arm $i\in\OPT$ and arm $j\in\OPTc$,
    	$\mu_i - \mu_j \ge (\Delta_i + \Delta_j)/2$.
    For any two arms $i,j\in\OPT$,
    	$\mu_i - \mu_j = \Delta_i - \Delta_j$.
\end{lemma}

\begin{proof}
    Recall that $\Mean{t}$ and $\Gap{t}$ denote the mean and the gap of
    the $t$-th largest arm, respectively.
    Then for any $i\in\OPT$ and $j\in\OPTc$,
    \[
        \mu_i - \mu_j 
    =   \Gap{i} + \Gap{j} - \Gap{K}
    \ge \Gap{i} + \Gap{j} - (\Gap{i} + \Gap{j}) / 2
    =   (\Gap{i} + \Gap{j}) / 2.
    \]
    Moreover, for any $i,j\in\OPT$,
        \[\mu_i - \mu_j = (\mu_i - \Mean{K+1}) - (\mu_j - \Mean{K+1}) = \Delta_i - \Delta_j.\]
\end{proof}

Now we prove Theorem~\ref{theo:LUCBr}.

\begin{proof}[Proof of Theorem~\ref{theo:LUCBr}]
    We define $\event$ as the event that for all arm $i\in\OPT$ and integer $j\ge 1$, it holds that $\left|\empmean{i}{j}-\mu_i\right| < U\left(j, \delta/(2K)\right)$. Moreover, for arm $i\in\OPTc$ and $j\ge 1$, $\left|\empmean{i}{j}-\mu_i\right| < U\left(j, \delta/[2(N-K)]\right)$. By Lemma~\ref{lem1} and a union bound, event $\event$ happens with probability at least
    \begin{align*}
    	& 1 - K\cdot\ceps\left(\delta/(2K)\right)^{1+\eps} - (N-K)\cdot\ceps\left(\delta/[2(N-K)]\right)^{1+\eps}\\
    	\ge& 1 - K\cdot\ceps\left(\delta/(2K)\right) - (N-K)\cdot\ceps\left(\delta/[2(N-K)]\right)\\
    	\ge& 1 - \ceps\delta.
   	\end{align*}
    It remains to prove the correctness and the sample complexity bound of \LUCBr{} conditioning on event $\event$.

    \paragraph{Conditional correctness.}
    Suppose for the purpose of contradiction the algorithm terminates at time $t$
    and returns $\High_t\ne\OPT$.
    In this case, there exists
    $i\in\High_t\cap\OPTc$ and $j\in\Low_t\cap\OPT$.
    Recall that $h_t$ is the arm in $\High_t$ with the lowest lower confidence bound 
	and $l_t$ is the arm in $\Low_t$ with the highest upper confidence bound.
    The definition of $h_t$ and event $\event$ guarantees that conditioning on $\event$,
    \[
        \mu_i > \empmean{i}{T_i(t)} - U\left(T_i(t), \delta/[2(N-K)]\right)
    \ge \empmean{h_t}{T_{h_t}(t)} - U\left(T_{h_t}(t), \delta/[2(N-K)]\right).
    \]
    Similarly,
    \[
        \mu_j < \empmean{j}{T_j(t)} + U\left(T_j(t), \delta/(2K)\right)
    \le \empmean{l_t}{T_{l_t}(t)} + U\left(T_{l_t}(t), \delta/(2K)\right).
    \]
    The stopping condition at round $t$ implies that
    \[
        \empmean{h_t}{T_{h_t}(t)} - U\left(T_{h_t}(t), \delta/[2(N-K)]\right)
    \ge \empmean{l_t}{T_{l_t}(t)} + U\left(T_{l_t}(t), \delta/(2K)\right).
    \]
    The three inequalities above together yield $\mu_i > \mu_j$, which contradicts the assumption that $i\in\OPTc$ and $j\in\OPT$.
    Thus, \LUCBr{} outputs the correct answer conditioning on event $\event$.

    \paragraph{Sample complexity bound.}
    We upper bound the sample complexity of \LUCBr{}
    by means of a \emph{charging argument}. 
    We define the \emph{critical arm} at time $t$, denoted by $\critical_t$,
    as the arm that has been pulled fewer times between $h_t$ and $l_t$,
    i.e., $\critical_t=\argmin_{i\in\{h_t, l_t\}}T_i(t)$.
    We ``charge'' the critical arm a cost of $1$, no matter whether it is actually pulled at this time step.
    It remains to upper bound the total cost that we charge each arm.
    To this end, we prove the following two claims:
    \begin{itemize}
        \item Once an arm has been sampled a certain number of times,
        it will never be critical in the future.
        \item The expected number of samples drawn from an arm
        is lower bounded by the total cost it is charged.
    \end{itemize}
    This directly gives an upper bound on the cost that we charge each arm,
    and thus an upper bound on the total sample complexity.

    \paragraph{First claim: no charging after enough samples.}
    For a fixed time step $t$, define
    	$\muhat_i = \empmean{i}{T_i(t)}$
    and
    	$r_i = U\left(T_i(t), \delta/(2N)\right).$
    Since $\delta/(2N)$ is smaller than $\delta/(2K)$ and $\delta/[2(N-K)]$,
    $r_i$ is greater than or equal to both
    	$U(T_i(t), \delta/(2K))$
    and
    	$U(T_i(t), \delta/[2(N-K)]).$
    It follows that conditioning on event $\event$,
    $|\muhat_i-\mu_i|<r_i$ holds for every arm $i$.
    
    In the following, we show that $r_{\critical_t}\ge\Delta_{\critical_t}/8$.
    In other words, let $\tau_i$ denote the smallest integer such that
        $U\left(\tau_i, \delta/(2N)\right) < \Delta_i / 8$.
    Then once arm $i$ has been sampled $\tau_i$ times,
    it will never become critical later.

    We prove the inequality in the following three cases separately.

    \textbf{Case 1. $h_t\in\OPTc$, $l_t\in\OPT$.}
    Since $h_t\in\High_t$ and $l_t\in\Low_t$, we have $\muhat_{h_t}\ge\muhat_{l_t}$. It follows that conditioning on event $\event$,
        $$\mu_{h_t} + r_{h_t}
        > \muhat_{h_t} \ge \muhat_{l_t}
        > \mu_{l_t} - r_{l_t},$$
    which implies that
        $$r_{h_t} + r_{l_t} > \mu_{l_t} - \mu_{h_t} \ge (\Delta_{h_t}+\Delta_{l_t})/2.$$
    The last step applies Lemma~\ref{lem:mean-vs-gap}.
    Recall that arm $\critical_t$ has been pulled fewer times than the other arm up to time $t$, and thus the arm has a larger confidence radius than the other arm. Then,
        $$r_{\critical_t} \ge (r_{h_t} + r_{l_t}) / 2
        \ge (\Delta_{h_t} + \Delta_{l_t}) / 4
        \ge \Delta_{\critical_t} / 4.$$

    \textbf{Case 2. $h_t\in\OPT$, $l_t\in\OPTc$.}
    Since the stopping condition of \LUCBr{} is not met, we have 
        \begin{equation*}\begin{split}
        \muhat_{h_t}-r_{h_t}
        &\le \muhat_{h_t} - U\left(T_{h_t}(t), \delta/[2(N-K)]\right)\\
        &< 	\muhat_{l_t} + U\left(T_{l_t}(t), \delta/(2K)\right)
        \le \muhat_{l_t}+r_{l_t}.
        \end{split}\end{equation*}
    It follows that conditioning on $\event$,
        $$\mu_{h_t}-2r_{h_t} < \muhat_{h_t}-r_{h_t} < \muhat_{l_t}+r_{l_t} < \mu_{l_t}+2r_{l_t},$$
    which implies that, by Lemma~\ref{lem:mean-vs-gap},
        $$r_{h_t} + r_{l_t} > (\mu_{h_t}-\mu_{l_t})/2\ge(\Delta_{h_t}+\Delta_{l_t})/4.$$
    Thus,
        $$r_{\critical_t}\ge(r_{h_t}+r_{l_t})/2\ge(\Delta_{h_t}+\Delta_{l_t})/8\ge\Delta_{\critical_t}/8.$$

    \textbf{Case 3. $h_t, l_t\in\OPT$ or $h_t, l_t\in\OPTc$.}
    By symmetry, it suffices to consider the former case.
    Since the arm $l_t$, which is among the best $K$ arms, is in $\Low_t$ by mistake,
    there must be another arm $j$ such that $j\in\OPTc\cap\High_t$.
    Recall that $h_t$ is the arm with the smallest lower confidence bound in $\High_t$. Thus we have
        \begin{equation*}\begin{split}
        \mu_{h_t} - 2r_{h_t} &< \muhat_{h_t} - r_{h_t}\\
        &\le \muhat_{h_t} - U\left(T_{h_t}(t), \delta/[2(N-K)]\right)\\
        &\le \muhat_{j} - U\left(T_{j}(t), \delta/[2(N-K)]\right)
        < \mu_j,
        \end{split}\end{equation*}
    and it follows from Lemma~\ref{lem:mean-vs-gap} that
        \begin{equation}\label{eq:case3-1}
            r_{h_t} > (\mu_{h_t}-\mu_{j}) / 2 \ge (\Delta_{h_t} + \Delta_j) / 4 \ge \Delta_{h_t}/4.
        \end{equation}
    Thus, if $\critical_t = h_t$, the claim directly holds. It remains to consider the case $\critical_t = l_t$.

    Since $\muhat_{h_t} \ge \muhat_{l_t}$, we have
        $$\mu_{l_t}-r_{l_t}<\muhat_{l_t}\le\muhat_{h_t}<\mu_{h_t}+r_{h_t}
        ,$$
    and it follows from Lemma~\ref{lem:mean-vs-gap} that
        \begin{equation}\label{eq:case3-2}
            r_{h_t} + r_{l_t} > \mu_{l_t} - \mu_{h_t} = \Delta_{l_t} - \Delta_{h_t}.
        \end{equation}
    Since we charge $\critical_t = l_t$, it holds that $r_{l_t} \ge r_{h_t}$, and thus by \eqref{eq:case3-1} and \eqref{eq:case3-2},
    \[
        r_{l_t} \ge (r_{h_t}+r_{l_t}) / 6 + 2r_{h_t} / 3
    >   (\Delta_{l_t} - \Delta_{h_t}) / 6 + \Delta_{h_t}/6
    =   \Delta_{l_t} / 6.
    \]

    This finishes the proof of the claim.

    \paragraph{Second claim: lower bound samples by costs.}
    We note that when we charge arm $\critical_t$ with a cost of $1$ at time step $t$,
    it holds that $T_{\critical_t}(t) \le (T_{h_t}(t) + T_{l_t}(t))/2$.
    According to \LUCBr{},
    arm $\critical_t$ is pulled at time $t$ with probability at least $1/2$.
    Recall that $\tau_i$ is defined as the smallest integer such that
        $U(\tau_i, \delta / (2N)) < \Delta_i / 8$.
    Let random variable $X_i$ denote that number of times
    that arm $i$ is charged before it has been pulled $\tau_i$ times.
    Since in expectation, an arm will get a sample after being charged at most twice,
    we have $\E[X_i] \le 2\tau_i$.

    By Lemma~\ref{lem2},
    \[
        \tau_i = O\left(\Delta_i^{-2}\left(\log\delta^{-1}+\log N + \log\log\Delta_i^{-1}\right)\right).
    \]
    Therefore, the sample complexity of the algorithm conditioning on event $\event$ is upper bounded by
    \[
        \sum_{i\in\inst}\E[X_i] = O\left(\sum_{i\in\inst}\tau_i\right)
    =   O\left(\sum_{i\in\inst}\Delta_i^{-2}\left(\log\delta^{-1} + \log N + \log\log\Delta_i^{-1}\right)\right).
    \]
\end{proof}

\section{lil'CLUCB Algorithm}\label{alg: lil'CLUCB}
    In this section, we give a formal description of the lil'CLUCB algorithm. The algorithm is shown in Algorithm~\ref{alg:lil-CLUCB}.

    \begin{algorithm2e}[H]
    \caption{$\lilCLUCB(\inst,K,\delta)$}
    \label{alg:lil-CLUCB}
    Initialization: Sample each arm once, and let $T_i(N)\gets 1$ for each $i\in\inst$\;
    \For{$t=N,N+1,...$} {
        $M_t \gets$ $K$ arms with the highest empirical means\;
        \lFor{$i \in \inst$} {
            $u_{i,T_i(t)} \gets U(T_i(t),\delta/N)$
        }
        \lFor{$i \in M_t$} {
            $\mutil_{i,T_i(t)} \gets \muhat_{i,T_i(t)}-u_{i,T_i(t)}$
        }
        \lFor{$i\in\inst\setminus M_t$} {
            $\mutil_{i,T_i(t)} \gets \muhat_{i,T_i(t)}+u_{i,T_i(t)}$
        }
        $\widetilde{M}_t$ $\gets$ the $K$ arms with the largest mean with respect to $\mutil$\;
        \lIf{$M_t=\widetilde{M_t}$}{\Return $M_t$}
        Sample $I_t \gets \argmax_{i\in M_t\symdiff\widetilde{M_t}}u_{i,T_i(t)}$\;
        \lFor{$i \in \inst$} {
            $T_i(t+1) \gets T_i(t) + \Ind{i = I_t}$
        }
    }
    \end{algorithm2e}

\section{Missing Proofs for \lilCLUCB{}}
\label{msproof: CLUCB}
In this section we show the proof for Theorem~\ref{theo:CLUCB}.
The theorem is a straightforward corollary of the following lemmas.
\begin{lemma}[Validity of Confidence Bound]
With probability at least $1-c_\epsilon\delta^{1+\epsilon}/N^\epsilon$, the confidence bound for each arm is valid, i.e., for each $i\in \inst$ and every $t\in\mathbb{N}$, we have
    \[|\mu_i-\muhat_{i,T_i(t)}|<U(T_i(t),\delta/N).\]
\end{lemma}
We denote the event that confidence bounds for all arms are valid as $\event$. For the following, we always condition on the event $\event$. It is not hard to see that given $\event$ holds, the algorithm always outputs the optimal solution.
\begin{lemma}[Correctness]
Given that the event $\event$ holds, the algorithm always outputs the correct answer.
\end{lemma}
\begin{proof}\quad Suppose that the algorithm terminates at time $t$. For the purpose of contradiction, we assume that the set $M_t$ is not the same as $\OPT$. Then there exists an arm $e\in M_t\setminus \OPT$, and an arm $e'\in \OPT \setminus M_t$. Since $M_t=\widetilde{M_t}$, $e\in \widetilde{M_t}\setminus \OPT$, and $e'\in \OPT \setminus \widetilde{M_t}$. So $\mutil_{e,T_e(t)}\geq\mutil_{e',T_{e'}(t)}$. But since $\mu_{e'}>\mu_{e}$, and due to validity of confidence bound, we have that $\mutil_{e',T_{e'}(t)}=\muhat_{e',T_{e'}(t)}+u_{e',T_{e'}(t)}>\mu_{e'}>\mu_e>\muhat_{e,T_{e}(t)}-u_{e,T_{e}(t)} = \mutil_{e,T_e(t)}$, which leads to a contradiction. Therefore, given $\event$ the algorithm always outputs the correct answer.\\
\end{proof}
Now we analyze the sample complexity of \lilCLUCB{}. The next lemma shows that whenever the confidence radius of an arm is sufficiently small, then that arm will not be further sampled.
\begin{lemma}
Given that the event $\event$ holds. For any arm $e\in \inst$, suppose $u_{e,T_e(t)}< \Delta_e/4$ at time $t$, then arm $e$ will not be sampled in that round. Notice that this is equivalent to saying that arm $e$ will not be sampled again, since the confidence radius is independent of time.
\end{lemma}
\begin{proof}
We prove this lemma by considering all possible cases seperately. Suppose arm $e$ is sampled at some time $t$, then by the sampling strategy, $e$ must be in the symmetric difference between $M_t$ and $\widetilde{M}_t$, i.e. $M_t\symdiff\widetilde{M}_t$. For simplicity of notation, we use $u_i$ to denote the confidence radius of arm $i$ at that moment. Then, $u_e\geq u_a,\forall a\in M_t\symdiff\widetilde{M}_t$. For the purpose of contradiction, we assume that $u_{e,T_e(t)}<\Delta_e/4$. We show that for each of the four cases below, we can obtain a contradiction: (1) $e\in\OPT\wedge e\in M_t\setminus\widetilde{M}_t$, (2) $e\in\OPT\wedge e\in \widetilde{M}_t\setminus M_t$, (3) $e\notin\OPT\wedge e\in M_t\setminus\widetilde{M}_t$ and (4) $e\notin\OPT\wedge e\in \widetilde{M}_t\setminus M_t$. \\
\\
\textbf{Case (1): } Suppose that $e\in\OPT\wedge e\in M_t\setminus\widetilde{M}_t$. Then there exists $e'\in\widetilde{M}_t$, s.t. $e'\notin\OPT$. If $e'\in \widetilde{M}_t\setminus M_t$, then $\mu_{e'}+2u_{e'}\geq \muhat_{e'}+u_{e'}\geq \muhat_e-u_e\geq \mu_e-2u_e$. Therefore, $\Delta_e>4\mu_e\geq \mu_e-\mu_{e'}\geq\Delta_e$, which is a contradiction. Therefore, $e'\in \widetilde{M}_t\cap M_t$. But then $\mu_{e'}\geq \muhat_{e'}-u_{e'}\geq\muhat_e-u_e\geq \mu_e-2u_e$, which again leads to a contradiction.\\
\\
\textbf{Case (2):} Suppose that $e\in\OPT\wedge e\in \widetilde{M}_t\setminus M_t$. Then there exists $e'\in M_t$, but $e'\notin\OPT$. Then $\muhat_{e'}\geq  \muhat_e$. If $e'\in M_t\setminus \widetilde{M}_t$, then $\mu_{e'}+u_{e'}\geq \mu_e-u_e$, which leads to $\Delta_e\leq \mu_e-\mu_{e'}\leq 2u_e<\Delta_e$, which is a contradiction. Therefore, it must be the case that each arm in $M_t\setminus\widetilde{M}_t$ belongs to $\OPT$, and we pick one of such arms $e''$. Then $\mu_{e'}\geq \muhat_{e'}-u_{e'}\geq \muhat_{e''}-u_{e''}\geq \muhat_e-u_{e''}\geq \mu_e-u_e-u_{e''}$. But since $u_e\geq u_{e''}$, then the above inequality implies $2u_e\geq \mu_e-\mu_{e'}$, which is a contradiction.\\
\\
\textbf{Case (3):} Suppose that $e\notin\OPT\wedge e\in M_t\setminus\widetilde{M}_t$. If there exists $e'\in\OPT$ and $e'\in \widetilde{M}_t\setminus M_t$, then $\mu_e+u_e\geq\muhat_e\geq \muhat_{e'}\geq \mu_{e'}-u_{e'}$. But $u_e\geq u_{e'}$, and therefore leads to a contradiction. So it must be the case that each arm in $\widetilde{M}_t\setminus M_t$ does not belong to $\OPT$. Therefore, there exists an arm $e'\in\OPT$ but $e'\notin M_t\cup \widetilde{M}_t$, and an arm $e''\in \widetilde{M}_t\setminus M_t$ such that $e''\notin\OPT$. Thus $\mu_e+u_e+u_{e''}\geq \muhat_e+u_{e''}\geq \muhat_{e''}+u_{e''}\geq \muhat_{e'}+u_{e'}\geq \mu_{e'}$, which leads to a contradiction $2u_e\geq \Delta_e$.\\
\\
\textbf{Cases (4):} Suppose that $e\notin\OPT\wedge e\in \widetilde{M}_t\setminus M_t$. Then there exists an arm $e'\in\OPT$ that is not in $\widetilde{M}_t$. If $e'\in M_t\setminus \widetilde{M}_t$, then by the construction of $\widetilde{M}_t$, we have that $\mu_e+2u_e\geq \muhat_e+u_e\geq \muhat_{e'}-u_{e'}\geq \mu_{e'}-2u_{e'}$. Since $u_e>u_{e'}$, it follows that $\Delta_{e}\leq \mu_{e'}-\mu_e\leq 4u_e<\Delta_e$, which is a contradiction. If $e'\notin M_t$, then $\mu_e+2u_e\geq \muhat_e+u_e\geq \muhat_{e'}+u_{e'}\geq \mu_{e'}$, which also leads to a contradiction.\\
\end{proof}

Using the preceding lemma, we can bound the number of samples taken by \lilCLUCB{}.
\begin{lemma}[Sample Complexity]
Given the event $\event$ holds, then for each arm $i$, the total number of samples taken from arm $i$ is bounded by
    $$O\left(\Delta_i^{-2}\left(\log\delta^{-1}+\log N+\log\log\Delta_i^{-1}\right)\right).$$
Therefore, the sample complexity of \lilCLUCB{} is
    $$O\left(\sum_{i\in \inst}\Delta_i^{-2}\left(\log\delta^{-1}+\log N+\log\log\Delta_i^{-1}\right)\right).$$
\end{lemma}
\begin{proof}
Let $T_i$ be the total number of samples obtained from arm $i$. By the proceding lemma and by putting all the constant factors together, there exists some constant $C_1$, such that for each arm $i\in \inst$, such that the following inequality holds.
$$
\sqrt{\frac{1}{T_i}\log(N\log(T_i)/\delta)}\leq C_1\Delta_i
$$
Solving for this inequality, we have that there exists some constant $C_2$ such that
\[
    T_i\leq\frac{C_2}{\Delta_i^2}\log(N\log(N/\Delta_e^2\delta)/\delta)
=   O\left(\Delta_i^{-2}\left(\log\delta^{-1}+\log N+\log\log\Delta_i^{-1}\right)\right).
\]
The lemma follows by summing over all arms.
\end{proof}

\section{Generalization of \lilCLUCB{}}
\label{lil'CLUCB:generalization}
In this section, we give more details on the generalization of \lilCLUCB{} to the combinatorial pure exploration (CPE) setting. The description of our generalized \lilCLUCB{} is given in algorithm \ref{alg:lil-CLUCB-gen}. \lilCLUCB{} differs from \CLUCB{} in that it uses a time-independent confidence radius, which allows us to obtain a better sample complexity upper bound.

\begin{algorithm2e}[h]
\caption{(Generalized) $\lilCLUCB(\inst,K,\delta)$. Parameter: $\epsilon$}
\label{alg:lil-CLUCB-gen}
Initialization: sample each arm once, $T_i(N)\leftarrow 1$ for each $i\in\inst$\;
\For{$t=N,N+1,...$} {
    $\mathbf{\muhat} \gets$ the vector of empirical means for all arms\;
    $M_t$ $\leftarrow$ Oracle($\muhat$)\;
    \lFor{$i\in\inst$} {
        $u_{i,T_i(t)} \leftarrow U(T_i(t),\delta/N)$
    }
    \lFor{$i \in M_t$} {
        $\mutil_{i,T_i(t)} \leftarrow \muhat_{i,T_i(t)}-u_{i,T_i(t)}$
    }
    \lFor{$i\in\inst\setminus M_t$} {
        $\mutil_{i,T_i(t)} \leftarrow \muhat_{i,T_i(t)}+u_{i,T_i(t)}$
    }
    $\widetilde{M}_t$ $\leftarrow$ Oracle($\mutil$)\;
    \lIf{$M_t=\widetilde{M_t}$} {
        \Return $M_t$
    }
    Sample $I_t\leftarrow\argmax_{i\in M_t\symdiff\widetilde{M_t}}u_{i,T_i(t)}$\;
    \lFor{$i\in\inst$} {
        $T_i(t+1)\leftarrow T_i(t) + \Ind{i = I_t}$ for each $i\in\inst$
    }
}
\end{algorithm2e}

Now we give the proof for Theorem~\ref{theo:GCLUCB}, which is restated below.

\noindent\textbf{Theorem~\ref{theo:GCLUCB}} (restated) \textit{
    Suppose the reward distribution of each arm is $\sigma$-sub-Gaussian.
    For any $\delta\in(0,1)$ and
    decision class $\mathcal{M}\subseteq 2^{[N]}$,
    with probability at least $1-c_\epsilon\delta^{1+\epsilon}/N^\epsilon$,
    \lilCLUCB{} outputs the correct answer and takes at most
    \[
        O\left(\width(\mathcal{M})^2\sigma^2\sum_{i\in \inst}\Delta_i^{-2}\left(\log\delta^{-1}+\log N+\log\log\Delta_i^{-1}\right)\right)
    \]
    samples.
}
\begin{proof}
    By Lemma~\ref{lem1} and the stopping condition of \lilCLUCB{},
    \lilCLUCB{} outputs the correct answer
    with probability at least $1-c_\epsilon\delta^{1+\epsilon}/N^\epsilon$.
    By~\cite[Lemma~10]{chen2014combinatorial},
    once the confidence radius of an arm $i$ is smaller than
        $\Delta_i/(3\width(\mathcal{M}))$,
    that arm will no longer be pulled.
    Combined with Lemma~\ref{lem2},
    the total number of samples obtained from arm $i$ is at most
    \[
        O\left(\width\left(\mathcal{M}\right)^2\sigma^2\Delta_i^{-2}\left(\log\delta^{-1}+\log N+\log\log\Delta_i^{-1}\right)\right).
    \]
    Summing up over all arms proves the theorem.
\end{proof}

\end{document}